%% file: root.tex
\documentclass[conference,12pt,onecolumn]{ieeeconf}

\IEEEoverridecommandlockouts                              

\usepackage[utf8]{inputenc} 
\usepackage[T1]{fontenc}    
\usepackage[hidelinks]{hyperref} 
\usepackage{url}            
\usepackage{booktabs}       
\usepackage{amsfonts}       
\usepackage{nicefrac}       
\usepackage{microtype}      
\usepackage{lipsum}
\usepackage{graphicx}
\usepackage{comment}
\usepackage{cite}
\usepackage{caption}
\usepackage{subcaption}	
\graphicspath{ {./images/} }
\usepackage{svg}
\usepackage{colortbl}
\usepackage{algorithm}
\usepackage{algorithmic}

\usepackage{amsmath}
\usepackage{cleveref}
\usepackage{mathrsfs}
\usepackage[english]{babel}
\usepackage{amssymb}

\newtheorem{theorem}{Theorem}

\newtheorem{remark}{Remark}

\def\R{\mathbb{R}}

\def\S{\mathcal{S}}
\def\G{\mathcal{G}}
\def\T{\mathcal{T}}
\def\U{\mathcal{U}}
\def\ra{\rightarrow}
\def\E{\mathbb{E}}

\begin{document}
\title{Approximate Linear Programming for Decentralized  Policy Iteration in Cooperative Multi-agent Markov Decision Processes}
\author{Lakshmi Mandal$^1$\thanks{$^1$Department of Computer Science and Automation, Indian Institute of Science, Bangalore 560012. E-mail: \{lmandal,shalabh\}@iisc.ac.in}, Chandrashekar Lakshminarayanan$^2$\thanks{$^2$Department of Computer Science and Engineering, Indian Institute of Technology Madras, Chennai 600036. E-mail: chandrashekar@cse.iitm.ac.in}, and Shalabh Bhatnagar$^1$\thanks{$^\star$S.Bhatnagar was supported in this work through a J.~C.~Bose Fellowship of SERB, Project No.~DFTM/ 02/ 3125/M/04/AIR-04 from DRDO under DIA-RCOE and the Walmart Centre for Tech Excellence, IISc.}}
  

\maketitle

\input{abstract}


\IEEEpeerreviewmaketitle

\input{introduction}

\input{problem_statement}

\input{proposed_algorithm}

\input{experiments}

\input{conclusions}

\bibliographystyle{IEEEtran}
\bibliography{MARL_PI_LP}


\end{document}

%% file: abstract.tex
\begin{abstract} 
In this work, we consider a ‘cooperative’ multi-agent Markov decision process (MDP) involving $m (>1)$ agents. At each decision epoch, all the $m$ agents independently select actions in order to maximize a common long-term objective. In the policy iteration process of multi-agent setup, the number of actions grows exponentially with the number of agents, incurring huge computational costs. Thus, recent works consider decentralized policy improvement, where each agent improves its decisions unilaterally, assuming that the decisions of the other agents are fixed. However, exact value functions are considered in the literature, which is computationally expensive for a large number of agents with high dimensional state-action space. Thus, we propose approximate decentralized policy iteration algorithms, using approximate linear programming with function approximation to compute the approximate value function for decentralized policy improvement. Further, we consider (both) cooperative multi-agent finite and infinite horizon discounted MDPs and propose suitable algorithms in each case. Moreover, we provide theoretical guarantees for our algorithms and also demonstrate their advantages over existing state-of-the-art algorithms in the literature.
\end{abstract}

%% file: introduction.tex
\section{Introduction}
\label{introduction}

Markov Decision Process (MDP) provides a mathematical framework for modeling problems involving sequential decision-making under uncertainty arising in fields such as artificial intelligence, finance, management, manufacturing and adaptive control \cite{bertsekas2017}. Here, the goal is to choose actions in response to the states of the system so as to maximize a certain long-term objective. Based on the number of decision-makers (also known as agents) being one or more, MDPs can be classified as `single-agent' MDP (SA-MDP) or `multi-agent' MDP (MA-MDP). 
Based on the objective of the individual agents, MA-MDP can be further classified as `competitive' or `cooperative'. In competitive MA-MDP, the $m$ agents ($m>1$) in general have competing objectives that can often be in conflict with those of other agents \cite{Sayin_2021, ramponi2021learning}.
In cooperative MA-MDP, all agents cooperate in order to maximize a common objective function \cite{correa2021cooperative,rizk2018decision}.

Dynamic programming (DP) methods, namely value iteration, policy iteration, and linear programming, are three basic numerical solution methods for MDPs \cite{Sutton2018, Li2022}. These are exact and full-state methods to find the optimal value and policy functions. However, MDPs arising in practice have a large number of states, and by the curse-of-dimensionality (COD), the number of states grows exponentially in the number of state variables. Thus, DP methods are impractical in the setting of MA-MDP with high-dimensional state space due to the excessive computational costs.
Approximate dynamic programming (ADP) methods \cite{powell2007approximate, wang2020mdp} overcome COD by approximating value function via a parameterized family of functions and finding the optimal parameter whose dimension is much smaller in comparison to the number of states.

Policy improvement is a key step in both exact as well as approximate DP methods and maps a candidate value function to its corresponding `one-step' greedy policy. In the case of MA-MDP, an important issue (in addition to COD) is that the overall action is a Cartesian product of the actions of the individual agents and hence the number of actions also grows exponentially in the number of agents \cite{Campbell2013}. Thus even if the number of states is small, policy improvement could be a computationally expensive operation due to the large number of actions. To address this issue, decentralized policy improvement is proposed in recent literature \cite{Bertsekas2021}, \cite{bertsekas2020}[Chap. 5]. In these works, the exact value function is centrally calculated, while policy improvement is decentralized, whereby each agent independently improves its own policy, assuming policy of the other agents is fixed. However, in the MA-MDP settings, computing the exact value function centrally is computationally expensive. In \cite{Chen2022} fully decentralized multi-agent algorithm, with function approximation, on infinite horizon discounted MDP is considered. 
In the SA-MDP setting, approximate linear programming (ALP) approaches are well studied in the literature \cite{CLSBCS}.
However, in MA-MDP, the ALP-based approach to approximate decentralized policy iteration is not considered in the literature. 

Thus, we propose an ALP-based method for decentralized policy iteration (called approximate decentralized policy iteration (ADPI)) in the cooperative MA-MDP setting. 
The following are the broad contributions of our work:\\
\textbf{(1)} We propose approximate decentralized policy iteration (ADPI) algorithms for MA-MDPs that can work under large state-action settings under (i) finite horizon and (ii) infinite horizon MDPs. Further, unlike \cite{Bertsekas2021, bertsekas2020}, our proposed algorithms compute the approximate value function, resulting in computational savings.\\
\textbf{(2)} We provide a detailed analysis and theoretical guarantees for the proposed ADPI algorithms.\\
\textbf{(3)} We empirically demonstrate the effectiveness of our
proposed algorithms on standard multi-agent tasks where even with more number of agents and complex state-action spaces, our algorithms outperform   
 \cite{Bertsekas2021,Chen2022}.
The rest of the paper is organized as follows. Section \ref{problem_formulation} describes the Background. Section \ref{sec:proposedAlgo} presents the proposed algorithms as well as the theoretical analysis. The experimental setting and numerical results are provided in section \ref{experiments_results}. Finally, Section \ref{conclusions} provides the concluding remarks.

%% file: problem_statement.tex
\section{Background}
\label{problem_formulation}
\subsection{Single Agent Markov Decision Process}
\label{subsec:RL}
Markov Decision Process (MDP) is a mathematical framework to cast problems involving sequential decision making under uncertainty. Formally, an MDP is given by the $4$-tuple $\langle \S, \U, g, P\rangle$, where $\S$ is the state space, $\U$ is the action space, $g$ is the single-stage cost and $P$ denotes the transition probability kernel. In this paper, we consider MDP with finite state-action space. Here, for $x,y\in \S$ and $u\in \U(x)$, $p_{xy}(u)$ is the probability of transitioning from state $x$ to $y$ while choosing action $u$ and $g_{xy}(u)$ denotes the associated cost.


\textbf{Policy and Value:} The behaviour of the agent is formally described by a `policy' denoted by $\pi=\{ \mu_0,\mu_1,\dots,\mu_k,\dots \} $, where, $\mu_k :\S \rightarrow \U$, $k =0,1,\dots$, are mappings from the state space $\S$ to the action space $\U$. The value function $J_\pi\colon\S\ra\R$ captures the long-term cumulative objective obtained by the agent acting according to a policy $\pi$ starting from each state. In this paper, we consider two kinds of long-term objectives: $(i)$ finite horizon total cost (FHTC)  and $(ii)$ infinite horizon discounted cost (IHDC), whose value functions are as follows.
\vspace{-10pt}
\begin{align}
    (\text{FHTC}):J_{\pi}(x) = \E [\sum_{k=0}^{N} g_{x_kx_{k+1}}(\mu_k(x_{k})) |x_0=x],\\
   (\text{IHDC}):J_{\pi}(x) = \E[ \sum_{k=0}^{\infty} \alpha^k g_{x_kx_{k+1}}(\mu_k(x_{k}))|x_0=x],
 \vspace{-10pt}  
\end{align}
where $\alpha \in(0,1)$ is the discount factor. In what follows, in this section, we briefly discuss the basic ideas using the IHDC setting (those related to the FHTC setting are similar and can be found in \cite{bertsekas2017}). 

\textbf{Policy Improvement (PI):} Given an MDP, we are interested in finding the optimal policy $\pi_*$ and its associated optimal value function $J_*$. The `one-step' greedy operator also known as the \emph{Bellman Operator} plays a central role in the solution methods for obtaining the optimal quantities. In the IHDC case, for a given $J\colon \S\ra\R$, the Bellman operator and its associated `one-step' policy improvement operator are, respectively, as follows.
\vspace{-5pt}
\begin{align}
    \label{Bellman_T}
    (\T J)(x) = \min_{u\in \U(x)}\sum_{y}p_{xy}(u)\left[g_{xy}(u)+\alpha J(y)\right],\\
    \label{PI}(\G J)(x) = \arg\min_{u\in \U(x)}\sum_{y}p_{xy}(u)\left[g_{xy}(u)+\alpha J(y)\right].
\end{align}
Further, for the IHDC case, for a given policy $\mu$, $J\colon \S\ra\R$, the Bellman operator $\T_\mu$ is defined as follows:
\vspace{-8pt}
\begin{align}
\vspace{-10pt}
    \label{Bellman_T_mu}
    (\T_{\mu} J)(x) = \sum_{y}p_{xy}(\mu(x))\left[g_{xy}(\mu(x))+\alpha J(y)\right].
\end{align}
\textbf{Dynamic Programming (DP) Approaches:} The basic numerical solution approaches to calculate $\pi_*$ and $J_*$ are value iteration, policy iteration and linear programming. The first two are well-studied approaches, and most RL algorithms are based on one of these. 
Linear programming casts the Bellman equation $J_*= \T J*$ as
\vspace{-8pt}
\begin{align*}
\vspace{-8pt}
    &J_* = \arg\max_{J\in\R^{|\S|}} c^\top J\\
          &s.t\quad J(x) \leq \sum_{y}p_{xy}(u)\left[g_{xy}(u)+\alpha J(y)\right], \forall x\in \S, u\in \U(x),
          \vspace{-15pt}
\end{align*}
and obtains the optimal policy as $\pi_*= \G J_*$.

\textbf{Approximate Dynamic Programming (ADP) Methods:} 
Most MDPs arising in practical applications have a large number of states. Hence, it is computationally expensive to calculate $J_*$ and $\pi_*$, and DP-based approaches cannot be directly used due to the COD problem. ADP methods handle the issue of large number of states by introducing function approximation, i.e., they approximate the value function $J_{\pi}\approx \Phi r$, where $\Phi \in  \R^{|\S|\times d}$ is a \emph{feature matrix} and $r\in\R^d$ is a parameter (with $d<<|\S|$). Approximate value iteration, approximate policy iteration, and approximate linear programming are some of the widely studied ADP methods.

\subsection{Multi-Agent Markov Decision Process}
In this paper, we consider the `cooperative' multi-agent MDP (CO-MA-MDP) with `$m$' agents, where at each stage, the action is given by 
$u = (u^1,\dots,u^m)$, with $u^i$ being the action of agent $i$. Hence, the overall action space is given by $\U = \U^1 \times \U^2 \times \dots \times \U^m$, where $\U^i$ is the action space of agent $i$ and, $u^i\in \U^i(x), x \in S$. As a result, the size of the action space grows exponentially in the number of agents. To see its impact, note that while the underlying equations are still \eqref{Bellman_T}-\eqref{Bellman_T_mu} in the case of the CO-MA-MDP, 
in the minimization in these equations, the action $u$ now has $m$ components each corresponding to an agent as opposed to a single component in equations \eqref{Bellman_T} and \eqref{PI}. Further, in \eqref{Bellman_T_mu}, the policy $\mu$ in the MA case has $m$ components $\mu=(\mu^1,\ldots,\mu^m)$, with each component corresponding to an agent as opposed to \eqref{Bellman_T_mu}. 
\subsection{Prior Work: Decentralized Policy Iteration}
Due to the fact that the action $u=(u^1,\ldots,u^m)$ has $m$ components, it becomes expensive to exactly compute the minimization in equation \eqref{Bellman_T}-\eqref{PI}. In order to address this issue, \cite{Bertsekas2021} proposed an `agent-by-agent' decentralized policy improvement (DPIm), wherein,  each agent improves its own policy component, assuming that the other agents continue working with their policy components fixed. In particular, when the $i^{\text{th}}$ agent improves its policy, agents $1,2,\dots, i-1$ would have already updated their policies in the current cycle and would be choosing actions according to their updated policies $\tilde{\mu}^{1},\tilde{\mu}^{2},\ldots,\tilde{\mu}^{i-1}$, respectively, whereas agents $i+1,i+2,\dots,m$ would not have updated their policies until then and would continue to pick actions as per the \emph{base} policy namely $\mu^{i+1},\mu^{i+2},\ldots,\mu^{m}$ respectively. The DPIm in \cite{Bertsekas2021} is shown in \Cref{algo:DPI}, where $\tilde{\mu}_{1:i-1}=(\tilde{\mu}^{1},\ldots,\tilde{\mu}^{i-1})$ denotes collectively the updated policies of agents $1,\ldots,i-1$ and $\mu_{i+1:m}=(\mu^{i+1},\ldots,\mu^{m})$ denotes collectively the \emph{base} policy of agents $i+1,\ldots,m$.
As discussed in \cite{Bertsekas2021}, the DPIm in \Cref{algo:DPI} requires a form of \emph{weak} communication between agents, wherein agent $i$ can obtain the knowledge of $\tilde{\mu}^j, j=1,\ldots,i-1$ (to achieve this, it is enough if agent $i-1$ can communicate with agent $i$). However, in the literature, DPIm with exact value function is considered, which is computationally expensive with a large number of agents and/or with large state-action space \cite{wang2020mdp}.
\begin{algorithm}
\caption{Decentralized Policy Improvement (DPIm): \\DPIm($i, \tilde{\mu}_{1:i-1},\mu_{i+1:m},J, \alpha$)}\label{algo:DPI}  \begin{algorithmic}[1]
    \STATE Input: agent $i$; updated policies $\tilde{\mu}_{1:i-1}$; base policies $\mu_{i+1:m}$; value function $J$
    \STATE $\widetilde{\mu}^i(x) \in {\arg \min}_{u^i \in \U^i(x)} \sum_{y \in S} p_{xy}({\widetilde{\mu}}^1(x),\dots,{\widetilde{\mu}}^{i-1}(x), $
    \\ $u^i, \mu^{i+1}(x),\dots,\mu^m(x))[g_{xy}({\widetilde{\mu}}^1(x),\dots,u^i,\dots,\mu^m(x)) $
    \\ $+\alpha J(y)], \forall x \in S $
    \STATE Return updated policy $\widetilde{\mu}^{i}$ for the agent $i$.
   \end{algorithmic} 
\end{algorithm}
\begin{algorithm}
\caption{Calculate approximate cost function: {CACFN}$(k, \widetilde{\pi},J^{\text{ALP}}_{k+1,\widetilde{\pi}})$}\label{CACFN_ALP_FH}
    \begin{algorithmic}[1]
      \STATE Input: $k,\widetilde{\pi},J^{\text{ALP}}_{k+1,\widetilde{\pi}}$ 
      \STATE \label{cost_fun_FH_ALP} $J_{k,\widetilde{\pi}}^{\text{ALP}}=\underset{r \in \mathbb{R}^{d}}{\max} ~c^\top \Phi r$, $c \in \mathbb{R}^{n}, \Phi \in \mathbb{R}^{n \times d}$
      \STATE \label{constrain_FH_ALP} Subject to: $\Phi r (x) \leq \sum_{{y} \in S} p_{xy}({\widetilde{\mu}_k}^1(x),\dots,{\widetilde{\mu}_k}^m(x))$\\
      $[g_{xy}({\widetilde{\mu}_k}^1(x),
     \dots,{\widetilde{\mu}_k}^{m}(x)) + J_{k+1,\widetilde{\pi}}^{\text{ALP}}(y)], \forall x \in S$
     \STATE Return $J_{k,\widetilde{\pi}}^{\text{ALP}}$
    \end{algorithmic}
\end{algorithm}
\begin{algorithm}
\caption{DPI using ALP for FH CO-MA-MDP}\label{alg:FH_PI_ALP}
\begin{algorithmic}[1]
\STATE \textit{Input:} $N,n,m$\\

$u_k=(u_k^1,u_k^2,\dots,u_k^m) \in \U_k^1(x) \times \U_k^2(x) \times \dots\times \U_k^m(x)$, for state $x \in S$. \\
$\pi=\{ \mu_0,\mu_1,\dots,\mu_k,\dots,\mu_{N-1} \} $,
where $\mu_k = (\mu_k^1,\mu_k^2,\dots,\mu_k^m)$.\\
\STATE Call {CACFN}$(k, {\pi},J^{\text{ALP}}_{k+1,{\pi}})$, i.e., \Cref{CACFN_ALP_FH} to get $J_{k,\pi}^{\text{ALP}}$.
\STATE Terminal cost $J_{N,\widetilde{\pi}}^{\text{ALP}}=J_{N,\pi}^{\text{ALP}}=g_N$.
 
\FOR {stages $k=N-1,N-2,\dots,0$}
    \REPEAT
        \FOR {agent $i=1,2,\dots,m$}
            \STATE Call \Cref{algo:DPI} to get $\widetilde{\mu}_k^i$.
        \ENDFOR
        \STATE Call {CACFN}$(k, \widetilde{\pi},J^{\text{ALP}}_{k+1,\widetilde{\pi}})$, i.e., \Cref{CACFN_ALP_FH} to get $J_{k,\widetilde{\pi}}^{\text{ALP}}$.
        \STATE $\pi=\widetilde{\pi}$
    \UNTIL{$J_{k,\widetilde{\pi}}^{\text{ALP}}(x) \geq J_{k,\pi}^{\text{ALP}}(x), $ $\forall x \in S$}
    \STATE Return optimal policy function for the $k^{\text{th}}$ stage.
\ENDFOR
\end{algorithmic}
\end{algorithm}

%% file: proposed_algorithm.tex
\section{Proposed Algorithms and Theoretical Analysis}
\label{sec:proposedAlgo}

This section describes our proposed methodology, i.e., approximate linear programming (ALP) for decentralized policy iteration (DPI) in cooperative multi-agent MDPs (CO-MA-MDPs). In this work, we consider finite horizon as well as infinite horizon CO-MA-MDPs. 
The proposed algorithms are written in a modular manner for ease of understanding.

\subsection{DPI using ALP in Finite Horizon CO-MA-MDPs}
This section demonstrates the proposed algorithm (\Cref{alg:FH_PI_ALP}) for finite horizon CO-MA-MDPs. In \Cref{alg:FH_PI_ALP}, ALP is implemented in decentralized policy iteration, where approximate policy evaluation is introduced unlike in \cite{Bertsekas2021}. 

The inputs in \Cref{alg:FH_PI_ALP} are the finite horizon length $N$, the number of states $n$ and the number of agents $m$. Further, the state space $S$, action space $\U$ and base policy $\pi$ are also taken as inputs. We denote the action space at the $k^{\text{th}}$ stage as $\U_k=\U_k^1\times \U_k^2 \times \dots\times \U_k^m$, where the action component for the $i^{\text{th}}$ agent is $u^i_k \in \U^i_k$. At first, we call {CACFN}$(k, \pi,J^{\text{ALP}}_{k+1,\pi})$ from \Cref{CACFN_ALP_FH} to get the cost function value (at the $k^{\text{th}}$ stage) $J^{\text{ALP}}_{k,\pi}$. This algorithm takes the current stage number $k$, current policy $\pi$ and the next stage cost $J^{\text{ALP}}_{k+1,\pi}$ as input and returns the cost at the $k^{\text{th}}$ stage. In this algorithm, the approximate cost is calculated using a linear program (see Lines \ref{cost_fun_FH_ALP} and \ref{constrain_FH_ALP} in \Cref{CACFN_ALP_FH}). In other words, we get the approximate cost $J^{\text{ALP}}_{k,\pi}$ as follows:
\begin{align}
\vspace{-15pt}
  & J_{k,{\pi}}^{\text{ALP}} =\underset{r \in \mathbb{R}^{d}}{\max} ~c^\top \Phi r, ~~ c\in \mathbb{R}^{n}, \Phi \in \mathbb{R}^{n \times d} \nonumber\\
    & \text{Subject to:}~ \Phi r (x) \leq \sum_{{y} \in S} p_{xy}({{\mu}_k}^1(x),\dots,{{\mu}_k}^m(x)) \nonumber\\
    &  [g_{xy}({{\mu}_k}^1(x),
     \dots,{{\mu}_k}^{m}(x)) + J_{k+1,{\pi}}^{\text{ALP}}(y)], \forall x \in S.
\end{align}

In this finite horizon MA-MDP algorithm, the terminal cost according to the base policy $\pi$, i.e., $J_{N,{\pi}}^{\text{ALP}}$ and the updated policy $\Tilde{\pi}$, i.e., $J_{N,\widetilde{\pi}}^{\text{ALP}}$ is the single stage cost $g_N$. As this is a finite horizon MDP, we perform decentralized policy iteration (DPI) using dynamic programming going backward in time, i.e., we first calculate the optimal policy for the ${(N-1)}$-stage problem using the fact that the optimal single-stage terminal cost is $g_N$. Next, for the ${(N-2)}$-stage problem, we use the optimal cost obtained for the ${(N-1)}$-stage problem and the process is repeated. At every stage, we repeat the $(i)$ Decentralized policy improvement, and $(ii)$ Approximate policy evaluation step (which is unlike \cite{Bertsekas2021,bertsekas2020}) until the policy converges, i.e., until the cost, according to the updated policy, $J^{\text{ALP}}_{k,\Tilde{\pi}}$ is greater than or equals the cost obtained according to the base policy, $J^{\text{ALP}}_{k,\pi}$. The procedure then returns the optimal policy for the $k^{\text{th}}$ stage. 

In the Decentralized policy improvement step, each agent $i$ updates its own policy component assuming that other agents are working as per their individual optimal policies. Now, $\tilde{\mu}_{1:i-1,k}=(\tilde{\mu}_{k}^{1},\ldots,\tilde{\mu}_{k}^{i-1})$ denotes collectively the updated policies of agents $1,\ldots,i-1$ and $\mu_{i+1:m,k}=(\mu_{k}^{i+1},\ldots,\mu_{k}^{m})$ denotes collectively the \emph{base} policy of agents $i+1,\ldots,m$ at the $k^{\text{th}}$ stage. In this step, for each agent $i$, 
\Cref{algo:DPI} is called, providing agent $i$, the current policy of other agents as mentioned above, and also the next stage cost $J^{\text{ALP}}_{k+1,\pi}$ as input, to obtain  the updated policy component $\widetilde{\mu}_k^i$ of the $i^{\text{th}}$ agent at the $k^{\text{th}}$ stage. In other words, we get the updated policy component of the $i^{\text{th}}$ agent in state $x$ as follows:
\vspace{-8pt}
\begin{align}
\vspace{-8pt}
\nonumber 
&\widetilde{\mu}^i(x) \in \underset{u^i \in \U^i(x)}{\arg \min} \sum_{y \in S} p_{xy}({\widetilde{\mu}}^1(x),\dots,{\widetilde{\mu}}^{i-1}(x), u^i, \dots,  \nonumber \\ 
 & \mu^m(x))[g_{xy}({\widetilde{\mu}}^1(x),\dots,u^i,\dots,\mu^m(x))+ \alpha J_{k+1,\pi}^{\text{ALP}}(y)]
 \label{imp}
\end{align}
Next, the updated policy $\Tilde{\pi}$ is evaluated by calling {CACFN}$(k, \widetilde{\pi},J^{\text{ALP}}_{k+1,\widetilde{\pi}})$, i.e.,  \Cref{CACFN_ALP_FH}, where current stage $k$, updated policy $\Tilde{\pi}$ and next stage cost 
$J^{\text{ALP}}_{k+1, \Tilde{\pi}}$ are provided as input. Next, the updated policy $\Tilde{\pi}$ is set as base policy $\pi$. After the completion of execution of \Cref{alg:FH_PI_ALP}, we receive the optimal policy and the corresponding optimal cost for all stages.
Now, we introduce \Cref{thm:FH_PI_ALP} to show a cost improvement property for \Cref{alg:FH_PI_ALP}. 

Let $\pi$ be the base policy and $\widetilde{\pi}$ be the corresponding roll-out policy obtained from Algorithm \ref{alg:FH_PI_ALP}.
Let,
\begin{equation}
\label{eq:FH_PI_ALP_1}
J_{k+1,\widetilde{\pi}}(x) \leq J_{k+1,\pi}(x) +\beta_{k+2}, ~\forall ~x \in S ~and ~k,
\end{equation}
with $\beta_{k+2}=\max_{x\in S} |Err_{k+2}(x)|$ as the largest error $Err_{k+2}(x): =J_{k+1, \pi}(x)-J^{\text{ALP}}_{k+1,\pi}(x)$, $x\in S$. Let 
$\beta :=\max\{\beta_N, \beta_{N-1} \dots\beta_{k+1},\dots\beta_1\}$. 
Rewriting (\ref{eq:FH_PI_ALP_1}), we obtain for all $x\in S$, $k=0,1,\ldots,N-1$
\vspace{-7pt}
\begin{eqnarray}
\label{eq:FH_PI_ALP_2}
J_{k+1,\widetilde{\pi}}(x) &\leq& J_{k+1,\pi}(x) +(N-k-1)\beta,\\
\label{eq:FH_PI_ALP_3}
J_{N,\widetilde{\pi}}(x) &=& J_{N,\pi}(x)=g_N(x).
\end{eqnarray}

\begin{theorem}\label{thm:FH_PI_ALP}
The following inequality holds:
\begin{align}
\label{eq:FH_PI_ALP_4}
J_{k,\widetilde{\pi}}(x) \leq J_{k,\pi}(x) +(N-k)\beta, ~\forall ~x \in S ~\text{and} ~k.
\end{align}
\end{theorem}
\begin{proof}   
For simplicity, we prove Theorem \ref{thm:FH_PI_ALP} for the two-agent case, i.e., $m=2$. The proof for an arbitrary number of agents $m > 2$  is entirely similar. We use mathematical induction for our proof.
Clearly, (\ref{eq:FH_PI_ALP_4}) holds for stage $j=N$, as $J_{N, \widetilde{\pi}} = J_{N, \pi}=g_N$. Thus, assuming 
that (\ref{eq:FH_PI_ALP_4}) holds for stage $j=k+1$, i.e., $J_{k+1 , \widetilde{\pi}}(x) \leq J_{k+1, \pi}(x) +(N-k-1)\beta$, $\forall x$, we show that it also holds for $j=k$, thereby completing the induction step. Thus, note that 
\begin{align*}
& J_{k,\widetilde{\pi}}(x)\\
&= \sum_{y \in S} p_{xy}(\widetilde{\mu}_k^1(x), \widetilde{\mu}_k^2(x))[g_{xy}(\widetilde{\mu}_k^1(x), \widetilde{\mu}_k^2(x))+J_{k+1,\widetilde{\pi}}(y)] \\
&\stackrel{(a)}{\leq} \sum_{y \in S} p_{xy}(\widetilde{\mu}_k^1(x), \widetilde{\mu}_k^2(x))[g_{xy}(\widetilde{\mu}_k^1(x), \widetilde{\mu}_k^2(x))+J_{k+1,\pi}(y)]  + (N-k-1)\beta \\
&= \sum_{y \in S} p_{xy}(\widetilde{\mu}_k^1(x), \widetilde{\mu}_k^2(x))[g_{xy}(\widetilde{\mu}_k^1(x), \widetilde{\mu}_k^2(x))+(J^{\text{ALP}}_{k+1,\pi}(y) +Err_{k+2}(y))] + (N-k-1)\beta\\
&\leq \sum_{y \in S} p_{xy}(\widetilde{\mu}_k^1(x), \widetilde{\mu}_k^2(x))[g_{xy}(\widetilde{\mu}_k^1(x), \widetilde{\mu}_k^2(x))+ J^{\text{ALP}}_{k+1,\pi}(y)] + (N-k)\beta
\\
&\stackrel{(b)}{=} \min_{u_k^2 \in \U_k^2(x)} \sum_{y \in S} p_{xy}(\widetilde{\mu}_k^1(x), u_k^2)[g_{xy}(\widetilde{\mu}_k^1(x), u_k^2)+ J^{\text{ALP}}_{k+1,\pi}(y)]  + (N-k)\beta
\end{align*}
Following similar steps for the $1^{\text{st}}$ agent as for the $2^{\text{nd}}$ agent,
\begin{align*}
  & J_{k,\widetilde{\pi}}(x) \leq \sum_{y \in S} p_{xy}({\mu}_k^1(x), \mu_k^2(x))[g_{xy}( {\mu}_k^1(x), \mu_k^2(x))+J^{\text{ALP}}_{k+1, \pi}(y)] + (N-k)\beta \\  
 &\stackrel{(c)}{\leq} \sum_{y \in S} p_{xy}({\mu}_k^1(x), \mu_k^2(x))[g_{xy}( {\mu}_k^1(x), \mu_k^2(x))+J_{k+1,\pi}(y)] + (N-k)\beta = J_{k, \pi}(x)+ (N-k)\beta
\end{align*}

In the above, $(a)$ follows from the induction hypothesis, $(b)$ follows from (\ref{imp}) and $(c)$ follows from the property of ALP solutions that $J^{\text{ALP}}_{k,\pi}(x)\leq J_{k,\pi}(x),\forall x\in S$. 
\end{proof}
\begin{remark}
\Cref{thm:FH_PI_ALP} states that the policy achieved at each stage in \Cref{alg:FH_PI_ALP} is close to the policy obtained when exact value function is computed \cite{Bertsekas2021} up to an error bound $(N-k)\beta$. The error here depends on $\beta$ and the time gap between stage $k$ and the end of horizon. From \Cref{thm:FH_PI_ALP}, we can say that near the end of horizon, the error goes to zero, and in other stages, the error depends on the function approximator. If the function approximator gives precise representations, then the error goes to zero, and we get an approximate policy that is as good as the exact policy. 
\end{remark}

\subsection{DPI using ALP in Infinite Horizon CO-MA-MDPs}
In \Cref{alg:IH_PI_ALP}, ALP is implemented in decentralized policy iteration (DPI) for the infinite horizon CO-MA-MDPs.
\begin{algorithm}
\caption{Calculate approximate cost function: {CACFN}$(\mu)$}\label{CACFN_ALP_IH}
    \begin{algorithmic}[1]
      \STATE Input: Policy $\mu$ 
      \STATE $r_\mu= \underset{{r\in\R^d}}{\arg\max} ~c^\top \Phi r $,
      Subject to: $\T_\mu \Phi r \geq \Phi r$
     \STATE Return $J_\mu^{\text{ALP}} = \Phi r_\mu$
    \end{algorithmic}
\end{algorithm}
\begin{algorithm}
\caption{DPI using ALP for IH CO-MA-MDP} \label{alg:IH_PI_ALP}
\begin{algorithmic}
\STATE \textit{Input:} Base policy $\mu_0=\left(\mu_0^1, \ldots,\mu_0^m\right)$; discount factor $\alpha$.
\STATE $J^{\text{ALP}}_{\mu_0}=\text{CACFN}(\mu_0)$ (call \Cref{CACFN_ALP_IH}) 
\FOR{$t=0,\ldots,T-1$}
\FOR{agent $i=1,2,\dots,m$}
\STATE $\widetilde{\mu}^i =\text{DPIm}(i,\tilde{\mu}_{1:i-1},\mu_{i+1:m},J^{\text{ALP}}_{\mu_t},\alpha)$ (call Algo.\ref{algo:DPI})
\ENDFOR
\STATE  $\mu_{t+1}=\widetilde{\mu}$.
\STATE 
$J^{\text{ALP}}= \text{CACFN}({\mu_{t+1}})$, (call \Cref{CACFN_ALP_IH}).
\ENDFOR
\STATE Return policy $\mu_{T}$.
\end{algorithmic}
\end{algorithm}
 In this algorithm, approximate policy evaluation is performed, unlike in \cite{Bertsekas2021,bertsekas2020}, where exact policy evaluation is done in a regular policy iteration procedure.  
In \Cref{alg:IH_PI_ALP}, the inputs are 
the state space $S$ (with $n$ states), action space $\U$, number of agents ($m$) and base policy $\mu=\{\mu^1,\mu^2, \dots, \mu^m \}$. Here $\U=\U^1\times \U^2 \times \dots \U^{i-1} \times \U^{i} \times \U^{i+1}, \dots\times \U^m$, where the action component for the $i^{\text{th}}$ agent in state $x \in S$ is $u^i \in \U^i(x)$. As this algorithm is designed for infinite horizon CO-MA-MDPs, we use a discount factor $\alpha \in (0,1)$.
At first, call \Cref{CACFN_ALP_IH} to get initial cost $J^{\text{ALP}}_{\mu}$ according to the base policy $\mu$. Then, we repeat the $i)$ decentralized policy improvement (DPIm) and $ii)$ approximate policy evaluation until we get the optimal policy, i.e., until the cost according to the updated policy, $J^{\text{ALP}}_{\Tilde{\mu}}$ is at least the cost according to the base policy, $J^{\text{ALP}}_{\mu}$ for all states $x \in S$. In the DPIm step, for each agent $i$, \Cref{algo:DPI} is called to get the updated policy $\Tilde{\mu}^i$ for the $i^{\text{th}}$ agent providing agent $i$, current policy $\mu$ and cost $J^{\text{ALP}}_{\mu}$ as input.
The policy update is thus carried out according to (\ref{imp}). 
In the approximate policy evaluation step, \Cref{CACFN_ALP_IH} is called to get the cost $J^{\text{ALP}}_{\Tilde{\mu}}$ according to the updated policy $\Tilde{\mu}$.
In \Cref{thm:IH_PI_ALP}, we show a cost improvement property of the updated policy $\Tilde{\mu}$ when using \Cref{alg:IH_PI_ALP}.

\begin{theorem}\label{thm:IH_PI_ALP}
In Algorithm \ref{alg:IH_PI_ALP}, let $ \beta_t:=\max_{x \in S}|J_{\mu_t}(x)-J^{\text{ALP}}_{\mu_t}(x)|$. We have\\
${\displaystyle
     J_{\mu_{t+1}}(x) \leq J_{\mu_t}(x) + \frac{1}{1- \alpha} \beta_t, \mbox{ }\forall x\in S}$.
\end{theorem}

\begin{proof}
For the sake of notational simplicity, we denote $\mu \stackrel{\triangle}{=} \mu_{t}$, $\tilde{\mu} \stackrel{\triangle}{=} \mu_{t+1}$ and $\beta \stackrel{\triangle}{=} \beta_t$, respectively. For ease of exposition, we consider the case of two agents (i.e., $m = 2$), as the case of an arbitrary number $m$ of agents with $m > 2$ is entirely analogous. 
For the case of  $m=2,\mu=\tilde{\mu}, J=J_{\mu}$, observe that 
for any $x\in S$,
\vspace{-5pt}
\begin{align*}
& \T _{\widetilde{\mu}}J_{\mu}(x) \\
&=\sum_{y=1}^n p_{xy}(\widetilde{\mu}_1(x),\widetilde{\mu}_2(x))(g_{xy}(\widetilde{\mu}_1(x),\widetilde{\mu}_2(x))+\alpha J_{\mu}(y)) \\
&\stackrel{(a)}{=} \sum_{y=1}^n p_{xy}(\widetilde{\mu}^1(x),\widetilde{\mu}^2(x))(g_{xy}(\widetilde{\mu}^1(x),\widetilde{\mu}^2(x))+\alpha (J^{\text{ALP}}_{\mu}(y)+\beta)) \\
&\stackrel{(b)}=\underset{u^2 \in \U^2(x)}{\min} \sum_{y=1}^n p_{xy}(\widetilde{\mu}^1(x),u^2)(g_{xy}(\widetilde{\mu}^1(x),u^2)+\alpha J^{\text{ALP}}_{\mu}(y)) + \alpha \beta\\
&\leq \sum_{y=1}^n p_{xy}(\widetilde{\mu}^1(x),\mu^2(x))(g_{xy}(\widetilde{\mu}^1(x),\mu^2(x))+\alpha J^{\text{ALP}}_{\mu}(y)) + \alpha \beta.
\vspace{-15pt}
\end{align*}
Following similar steps for the $1^{\text{st}}$ agent as the $2^{\text{nd}}$ agent,
\vspace{-5pt}
\begin{align*}
& \T _{\widetilde{\mu}}J_{\mu}(x) \leq \sum_{y=1}^n p_{xy}(\mu^1(x),\mu^2(x))(g_{xy}(\mu^1(x),\mu^2(x))+ \alpha J^{\text{ALP}}_{\mu}(y)) + \alpha \beta\\
&\stackrel{(c)}{\leq} \sum_{y=1}^n p_{xy}(\mu(x))(g_{xy}(\mu(x))+
\alpha J_{\mu}(y)) + \alpha \beta\\ 
&= J_{\mu}(x) + \alpha \beta.  
\end{align*}
In the above, $(a)$ follows from $J_{\mu}=J^{\text{ALP}}_{\mu}+\beta$, $(b)$ follows from (\ref{imp}) and $(c)$ follows from the property of the ALP solution that $J^{\text{ALP}}_{\mu}(x)\leq J_{\mu}(x),\forall x\in S$.
Hence, $\forall x\in S$, $e=(1,1,\ldots,1)^\top$ being the vector of all elements being 1,
\begin{align*}
& \T_{\widetilde{\mu}} J_{\mu}(x) \leq J_{\mu}(x) + \alpha \beta\\
&\Rightarrow \T^2_{\widetilde{\mu}} J_{\mu}(x)\leq \T_{\widetilde{\mu}}(J_{\mu}+ \alpha \beta e)(x)\\
&\stackrel{(d)}{\Rightarrow} \T^2_{\widetilde{\mu}} J_{\mu}(x) \leq \T_{\widetilde{\mu}}J_{\mu}(x) + \alpha ^2 \beta\\
&\Rightarrow \T^2_{\widetilde{\mu}} J_{\mu} (x)\leq J_{\mu}(x) + \alpha \beta + \alpha ^2 \beta.   
\end{align*}
In the above, $(d)$ follows from (see RHS above $(d)$)
\begin{eqnarray*}
&\T _{\widetilde{\mu}} (J_{\mu} + \alpha \beta e)(x)
= \sum_{y=1}^{n}p_{xy}((\widetilde{\mu}_1(x),\widetilde{\mu}_2(x))\\
&(g_{xy}(\widetilde{\mu}_1(x),\widetilde{\mu}_2(x))
  + \alpha (J_{\mu}(y)+ \alpha \beta))\\
&=\T _{\widetilde{\mu}} J_{\mu}(x) + \alpha^2\beta.
\end{eqnarray*}

Iterating the above step an infinite number of times, we obtain that
$J_{\widetilde{\mu}}(x) \leq J_\mu(x) + \frac{1}{1 - \alpha} \beta$, $\forall x\in S$.
  \end{proof}
\begin{remark}  
\Cref{thm:IH_PI_ALP} demonstrates that the policy obtained using function approximation in \Cref{alg:IH_PI_ALP}, is close to the policy obtained using exact value function in \cite{Bertsekas2021} up to an error bound of $\frac{1}{1 - \alpha} \beta$. 
\end{remark}

%% file: experiments.tex
\section{Experiments and Results}
\label{experiments_results}

In this section, we demonstrate the performance of the proposed algorithms, i.e., \Cref{alg:FH_PI_ALP} and \Cref{alg:IH_PI_ALP} on standard cooperative multi-agent MDP 
We implement both algorithms in the cost minimization as well as reward maximization applications (both cases are equivalent). For comparison with \Cref{alg:FH_PI_ALP} / \Cref{alg:IH_PI_ALP}, we also implement the ``Dynamic Programming (DP) / Policy Iteration (PI) for FH/IH CO-MA-MDP'' using exact value function \cite{Bertsekas2021}, and ``Fully Decentralized DP/PI for FH/IH CO-MA-MDP'' with function Approximation (Fully Dec.~DP/PI) \cite{Chen2022}. 

\textbf{1. Cost minimization application:} We consider the Flies-Spiders application \cite{Bertsekas2021}, but in a two-dimensional (2-D) grid world (i.e., App. 1). In this 2-D grid world with $16$ states, there are two spiders (agents) and two flies (goal states), where the position of the flies is fixed and the initial position of spiders is any random location in the grid. At each time instant, spiders can move one step in one of the four directions, i.e., ``up'', ``down'', ``left'', or ``right''. 
For the edge states in this grid world, an action that can take the agent out of the grid gets a penalty but doesn't change its location. In this application, the single-stage cost is $+1$, and penalty is $+2$ if spiders collide with each other, until both the flies are caught by the two spiders. 
In this application, in \Cref{alg:FH_PI_ALP}, we implement two different horizon lengths of $N=10$ and $N=15$, respectively, and the goal is to cover all the goal states at a minimum cost.  Further, for the infinite horizon discounted cost MA-MDP problem, we convert the problem into a continuing (non-episodic) task by randomly generating the initial positions of the spiders once they eat the flies and putting back the flies in the same positions as before. Here, we 
implement \Cref{alg:IH_PI_ALP} and our objective is to minimize the total expected discounted cost of the two agents. \Cref{tbl:result_comp_fly} presents the empirical results of our proposed algorithms and corresponding comparison results on the App. 1.

\begin{table}[ht]
\centering
\caption{Comparison Results in Flies-Spiders application.}
\label{tbl:result_comp_fly}
\vspace{-5pt}
\scalebox{1.2}{
\begin{tabular}{|c|c|c|}
\cline{1-3}
\textbf{Algorithm}  & \textbf{\# Iterations to converge} & \textbf{Time (Sec.)} \\ 
 \cline{1-3}
 \textbf{Dynamic Programming for FH CO-MA-MDP} \cite{Bertsekas2021} & 12 &   900.9    \\ \cline{1-3}
\textbf{DPI using ALP for FH CO-MA-MDP (Proposed)} & 6 &  \textbf{100.4}    \\ \cline{1-3}
 \cline{1-3}
 \textbf{Regular Policy Iteration for IH CO-MA-MDP} \cite{Bertsekas2021} & 200 &   485.3    \\ \cline{1-3}
\textbf{DPI using ALP for IH CO-MA-MDP (Proposed)} & 1 &  \textbf{25.6}    \\ \cline{1-3}
\end{tabular}}
\end{table}

\begin{figure}[htb]
\centering
\vspace{-44pt}
  \subfloat[]{
  \hspace{-15pt}
	\begin{minipage}[c][1\width]{
	   0.5\textwidth}
	   \centering
        \includegraphics[width=0.94\textwidth]{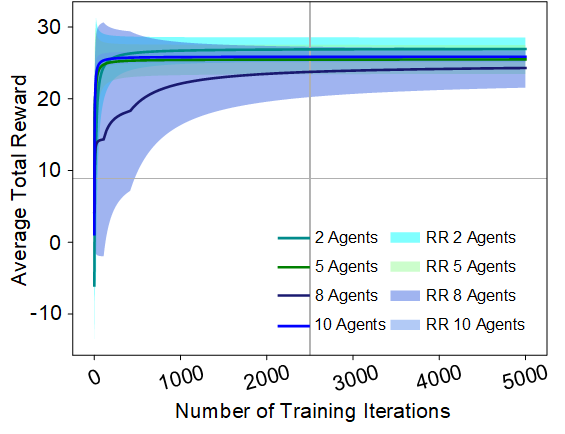}
      \vspace{-40pt}
	\end{minipage}}	
  \subfloat[]{
	\begin{minipage}[c][1\width]{
	   0.5\textwidth}
        \hspace{-25pt}
	   \centering
	   \includegraphics[width=0.94\textwidth]{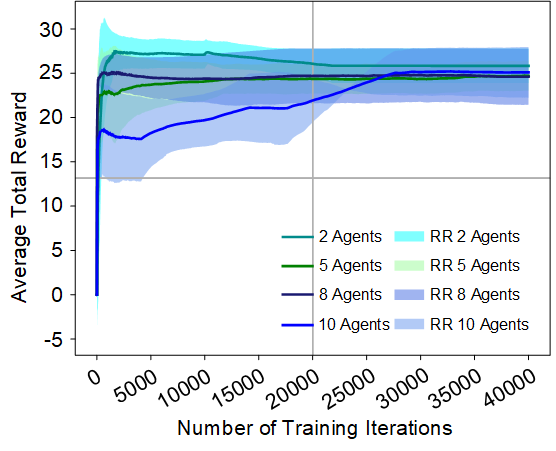}
   \vspace{-40pt}
	\end{minipage}}
\vspace{-10pt}
\caption{Performance in terms of average reward and standard deviation obtained from  $(a)$ \Cref{alg:FH_PI_ALP}$, (b) $\Cref{alg:IH_PI_ALP}, for different numbers of agents on App. 2.}
\label{plot_prop_algo}
\vspace{-5pt}
\end{figure}

\begin{figure}[htb]
\centering
\vspace{-40pt}
  \subfloat[]{
  \hspace{-15pt}
	\begin{minipage}[c][1\width]{
	   0.5\textwidth}
	   \centering
        \includegraphics[width=0.94\textwidth]{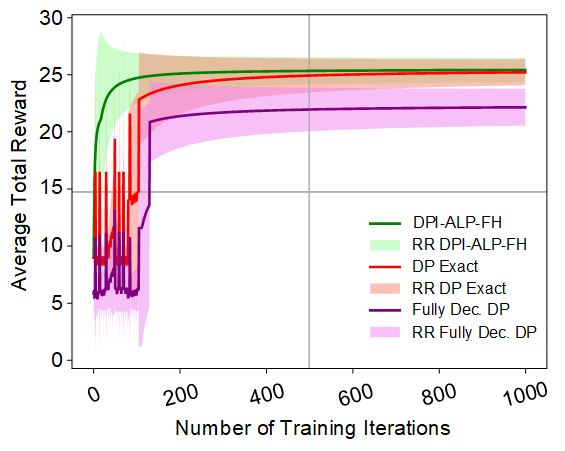}
       \vspace{-40pt}
	\end{minipage}}	
  \subfloat[]{
	\begin{minipage}[c][1\width]{
	   0.5\textwidth}
        \hspace{-25pt}
	   \centering
	   \includegraphics[width=0.94\textwidth]{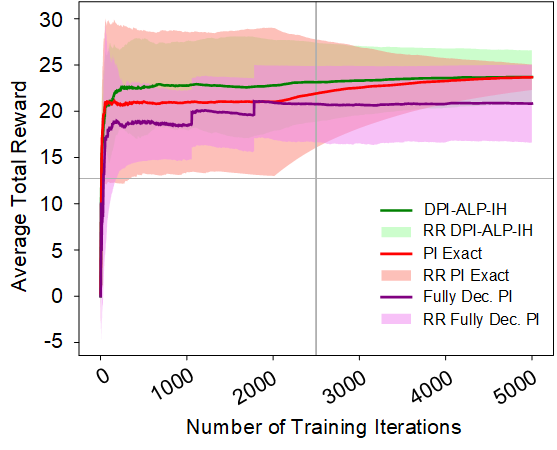}
   \vspace{-40pt}
	\end{minipage}}
\vspace{-10pt}
\caption{Performance comparisons ($m=5$) of $(a)$~\Cref{alg:FH_PI_ALP}$, (b)~  $\Cref{alg:IH_PI_ALP} with other algorithms on App. 2.}
\label{plot_comparison}
\end{figure}

\begin{figure}[htb]
\centering
\vspace{-50pt}
  \subfloat[]{
	\begin{minipage}[c][1\width]{
	   0.50\textwidth}
	   \centering
        \includegraphics[width=0.96\textwidth]{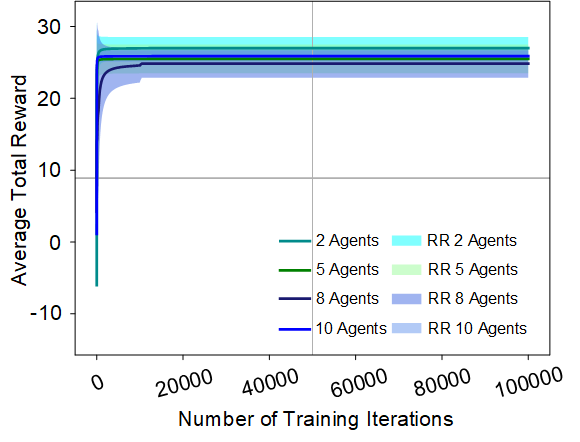}
        \vspace{-40pt}
	\end{minipage}}	
  \subfloat[]{
	\begin{minipage}[c][1\width]{
	   0.50\textwidth}
        \hspace{-30pt}
	   \centering
	   \includegraphics[width=0.96\textwidth]{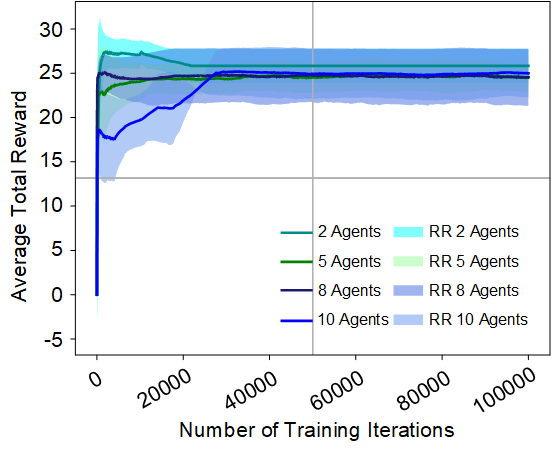}
    \vspace{-50pt}
	\end{minipage}}
 \vspace{-10pt}
\caption{Performance in terms of average reward and standard deviation obtained from  $(a)$ Algorithm 3, $(b)$ Algorithm 5 for different number of agents for 100,000 iterations on App. 2.}
\label{plot_prop_algo_100000}
\end{figure}

\begin{figure}[htb]
\centering
\vspace{-50pt}
  \subfloat[]{
	\begin{minipage}[c][1\width]{
	   0.50\textwidth}
	   \centering
        \includegraphics[width=0.96\textwidth]{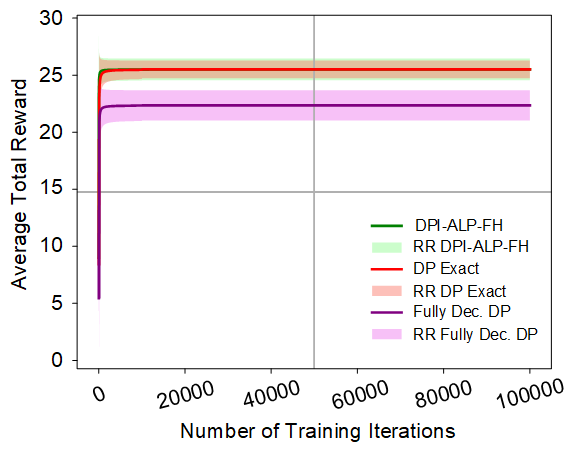}
        \vspace{-40pt}
	\end{minipage}}	
  \subfloat[]{
	\begin{minipage}[c][1\width]{
	   0.50\textwidth}
        \hspace{-30pt}
	   \centering
	   \includegraphics[width=0.96\textwidth]{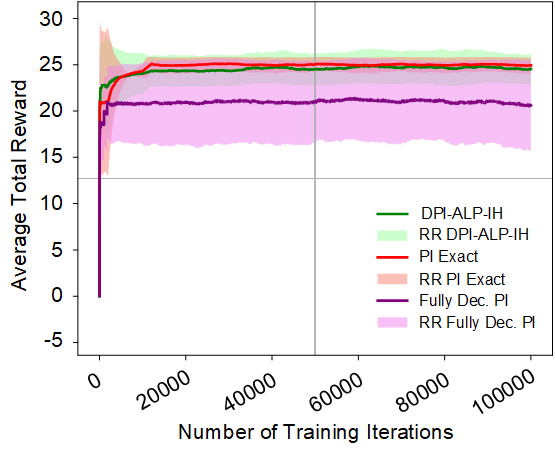}
    \vspace{-50pt}
	\end{minipage}}
 \vspace{-10pt}
\caption{{Performance comparisons ($m=5$) of $(a)$ Algorithm 3, $(b)$ Algorithm 5 with other algorithms for 100,000 iterations on App. 2.}}
\label{plot_comparison_100000}
\end{figure}

\begin{table}[htb]
\centering

\caption{Performance comparisons ($m=5$) of our proposed algorithms, i.e., Algorithm 3 and Algorithm 5, with \cite{Bertsekas2021},\cite{Chen2022}, averaged across 10 independent runs on App. 2.}
\label{tbl:result_comp_new}
\scalebox{1.2}{
\begin{tabular}{|c|c|c|}
\cline{1-3}
\textbf{Algorithm}  & \textbf{Mean $\pm$ Std.} & \textbf{\# Iterations to converge} \\ 
 \cline{1-3}
 \textbf{Dynamic Programming for FH CO-MA-MDP}\cite{Bertsekas2021} &  $25.5 \pm 0.78$ &  1000 \\ \cline{1-3}
\textbf{Fully Decentralized FH CO-MA-MDP} & $22.35 \pm 1.32$ &  150    \\ \cline{1-3}
\textbf{DPI using ALP for FH CO-MA-MDP (Proposed)} & \textbf{25.5 $\pm$ 0.95} &  \textbf{100}    \\ \cline{1-3}
 \textbf{Regular Policy Iteration for IH CO-MA-MDP} \cite{Bertsekas2021} & $24.94 \pm 0.86$ &    4500   \\ \cline{1-3}
\textbf{Fully Decentralized IH CO-MA-MDP} \cite{Chen2022} & $20.62 \pm 4.85$ & 1800     \\ \cline{1-3}
\textbf{DPI using ALP for IH CO-MA-MDP (Proposed)} & \textbf{24.56 $\pm$ 1.62} &  \textbf{200}    \\ \cline{1-3}
\end{tabular}}
\end{table}
\textbf{2. Reward maximization application:} We further extended our experiments to grid world environment (App. 2) with $2500$ states and $5$ actions, i.e., go one step ``up'', ``down'', ``left'', ``right'' or ``stay'' at the same location. In these experiments, our objective is to maximize the long-term collaborative reward for $m$ agents with $m=2,5,8,10$. The implementation environment of \Cref{alg:FH_PI_ALP}  has an equal number of goal states as the number of agents and when an agent reaches its goal state, it receives a positive reward, otherwise a negative reward. Here we set $N=50$. On the other hand, \Cref{alg:IH_PI_ALP}
has some randomly generated positive reward state, i.e., if any agent enters this state, it will get a positive reward, and the reward for all other states is $0$.

We obtain ten independent runs of the algorithms using different initial seeds. The average reward (mean) and standard deviation (std.) of rewards, respectively, obtained at each time instant from the ten runs, are calculated and plotted as a dark-colored line and a light-colored reward region (RR), respectively. The performance of our algorithms during training and the comparison with existing algorithms are graphically presented in \Cref{plot_prop_algo} and \Cref{plot_comparison}, respectively.  We run $100,000$ iterations of our proposed algorithms and the algorithms used for comparison. However, we show these plots for a lower number of iterations for better visualization of the speed of convergence of each algorithm. Nonetheless, we also show, in \Cref{plot_prop_algo_100000} and \Cref{plot_comparison_100000}, the plots over 100,000 runs of all algorithms. 
 In each figure, the $x$-axis represents the number of training iterations, and the $y$-axis represents the average total reward obtained so far (calculated as explained previously). Further, in \Cref{tbl:result_comp_new}, we summarize the final values of the mean and standard error of the average reward obtained from the ten independent runs after all the algorithms have been run for 100,000 iterations as well as the average number of iterations needed for each algorithm to converge. We show this as a representative table here for the case of $m=5$.
 We make the following observations from the experiments using our algorithms on the aforementioned applications (i.e., App. 1 and App. 2):\\
    \textbf{1.} Parts $(a)$--$(b)$ of \Cref{plot_prop_algo} show that the average total reward of \Cref{alg:FH_PI_ALP} and \Cref{alg:IH_PI_ALP}, respectively, converges for all four considered values of $m$.\\ 
    \textbf{2.} Tables \ref{tbl:result_comp_fly}--\ref{tbl:result_comp_new} and parts $(a)$--$(b)$ of \Cref{plot_comparison} demonstrate that our proposed algorithms  converge faster than those of \cite{Bertsekas2021} and \cite{Chen2022}. 
    \Cref{tbl:result_comp_fly} shows that our proposed algorithm for the finite horizon setting is nearly 9 times faster than DP. Further, our algorithm for the infinite horizon setting is more than 19 times faster than regular PI.\\
    \textbf{3.} Table \ref{tbl:result_comp_new}, parts $(a)$--$(b)$ of Fig. \ref{plot_comparison} show that the proposed algorithms achieve equal or higher average total reward compared to exact value \cite{Bertsekas2021}/fully dec. CO-MA-MDP\cite{Chen2022}.\\
    \textbf{4.} Because of function approximation in our algorithms, unlike \cite{Bertsekas2021}, when 10 agents ($m=10$), 2500 states ($\mid \S \mid=2500$), and 30 basis functions ($d=30$) are used, we get a dimensionality reduction factor $\frac{{\mid \S \mid} ^m}{d}= 3.2*10^{32}$.

%% file: conclusions.tex
\section{Conclusions}
\label{conclusions}
In this paper, we consider multi-agent Markov decision processes (MA-MDP) with cooperative agents under two different settings: (a) finite horizon and (b) infinite horizon, respectively. We separately present a novel algorithm for each of the two settings where we use {approximate linear programming for decentralized policy iteration in CO-MA-MDP}. We prove the convergence of both our proposed algorithms and show the applications of these to standard cooperative MA-MDP tasks. Our experimental results confirm that our algorithms outperform competing algorithms in the literature, involving significantly fewer iterations while providing better performance.
\vspace{-5pt}